\documentclass[runningheads]{llncs}
\usepackage[T1]{fontenc}
\usepackage{graphicx}
\usepackage{wrapfig}
\usepackage{amsmath}
\usepackage{amssymb}
\usepackage{amsfonts}
\usepackage{xcolor}
\usepackage{mathtools}
\usepackage{siunitx}
\usepackage{caption}
\usepackage{graphicx}
\usepackage{hyperref}
\usepackage{cleveref}
\usepackage[numbers,sort&compress]{natbib}
\usepackage[ruled,vlined,linesnumbered]{algorithm2e}

\usepackage{tikz}
\usetikzlibrary{arrows, positioning}

\crefname{figure}{Figure}{Figures}
\crefname{equation}{}{}
\crefname{table}{Table}{Tables}
\crefname{section}{Section}{Sections}
\crefname{algocf}{Algorithm}{Algorithms}
\Crefname{algocf}{Algorithm}{Algorithms}

\newcommand{\control}{\mathbf{u}}
\newcommand{\horizon}{T}
\newcommand{\state}{\mathbf{x}}
\newcommand{\noise}{\boldsymbol{\epsilon}}
\newcommand{\statenoise}{\noise^{\state}}
\newcommand{\paramsnoise}{\noise^{\params}}
\newcommand{\noisecov}{\mathbf{P}}
\newcommand{\paramsnoisecov}{\noisecov^{\params}}
\newcommand{\statenoisecov}{\noisecov^{\state}}
\newcommand{\paramsdynamics}{\mathbf{g}}

\newcommand{\learningfn}{\mathbf{\ell}}

\newcommand{\dynamics}{\mathbf{f}}
\newcommand{\dynjac}{\hat{\mathbf{F}}}

\newcommand{\kalmangain}{\mathbf{K}}
\newcommand{\innovationcov}{\mathbf{S}}
\newcommand{\innovation}{\mathbf{s}}
\newcommand{\dimstate}{n_{\state}}
\newcommand{\dimcontrol}{n_{\control}}
\newcommand{\dimparams}{n_{\params}}
\newcommand{\cost}{J}

\newcommand{\mutualinfo}{I}
\newcommand{\directedinfo}{I}
\newcommand{\entropy}{h}

\newcommand{\policy}{\mathbf{\pi}}
\newcommand{\timestart}{t}
\newcommand{\timeindex}{i}

\newcommand{\params}{\boldsymbol{\theta}}
\newcommand{\belief}{\boldsymbol{{\vartheta}}}
\newcommand{\beliefmean}{\boldsymbol{\bar{\belief}}}
\newcommand{\beliefcov}{\boldsymbol{\boldsymbol{\Sigma}}}
\newcommand{\beliefdynamics}{\hat{\learningfn}}
\newcommand{\beliefupdater}{\boldsymbol{b}}
\newcommand{\observation}{\mathbf{o}}
\newcommand{\observationfn}{\mathbf{q}}
\newcommand{\dimobs}{n_{\observation}}
\newcommand{\observationpredicted}{\hat{\observation}}
\newcommand{\obsjac}{\mathbf{Q}}

\newcommand{\beliefpredicted}{\hat{\belief}}
\newcommand{\statepredicted}{\hat{\state}}
\newcommand{\innovationcovpredicted}{\hat{\innovationcov}}

\newcommand{\paramsprocjac}{\mathbf{G}}
\newcommand{\noisepredicted}{\hat{\noise}}

\newcommand{\paramsQ}{\mathbf{Q}}
\newcommand{\paramsR}{\mathbf{R}}
\newcommand{\paramsA}{\mathbf{A}}
\newcommand{\paramsB}{\mathbf{B}}

\newcommand{\dimbelief}{n_{\belief}}

\usepackage{float}

\crefname{algocf}{Algorithm}{Algorithms}
\Crefname{algocf}{Algorithm}{Algorithms}

\ifdefined\assumption
\else
  \newtheorem{assumption}{Assumption}
\fi

\usepackage{xcolor}
\usepackage{tcolorbox}

\definecolor{myorange}{HTML}{FEAE00}

\newcounter{contributionnum}

\newcommand{\contribution}[2][]{%
  \refstepcounter{contributionnum}%
  \begin{tcolorbox}[
    colback=myorange!10, 
    colframe=myorange,
    title={\textbf{Contribution \thecontributionnum:} #1},
    coltitle=black,
    fonttitle=\bfseries
  ]
    #2
  \end{tcolorbox}%
}

\usepackage{color}

\begin{document}
\title{Generalized Information Gathering Under Dynamics Uncertainty}
%
%
\author{Fernando Palafox\inst{1}\orcidID{0009-0005-0836-4663} \and
Jingqi Li\inst{1}\orcidID{0000-0002-3731-3807} \and
Jesse Milzman\inst{2}\orcidID{0000-0003-4937-8912} \and
David Fridovich-Keil \inst{1}\orcidID{0000-0002-5866-6441}}
\authorrunning{F. Palafox et al.}
%
\institute{University of Texas at Austin, Texas, USA\\ \email{\{fernandpalafox,dfk\}@utexas.edu, jingqi.li@austin.utexas.edu} \and
DEVCOM Army Research Laboratory, New York, USA \\
\email{jesse.m.milzman.civ@army.mil}}
\maketitle              
%


\begin{abstract}
    An agent operating in an unknown dynamical system must learn its dynamics from observations. 
    Active information gathering accelerates this learning, but existing methods derive bespoke costs for specific modeling choices: dynamics models, belief update procedures, observation models, and planners. 
    We present a unifying framework that decouples these choices from the information-gathering cost by explicitly exposing the causal dependencies between parameters, beliefs, and controls. 
    Using this framework, we derive a general information-gathering cost based on Massey's directed information \cite{massey1990causality} that assumes only Markov dynamics with additive noise and is otherwise agnostic to modeling choices. 
    We prove that the mutual information cost used in existing literature is a special case of our cost. 
    Then, we leverage our framework to establish an explicit connection between the mutual information cost and information gain in linearized Bayesian estimation, thereby providing theoretical justification for mutual information-based active learning approaches. 
    Finally, we illustrate the practical utility of our framework through experiments spanning linear, nonlinear, and multi-agent systems.
\end{abstract}

\begin{figure}[ht]
\centering
\includegraphics[width=1.0\linewidth]{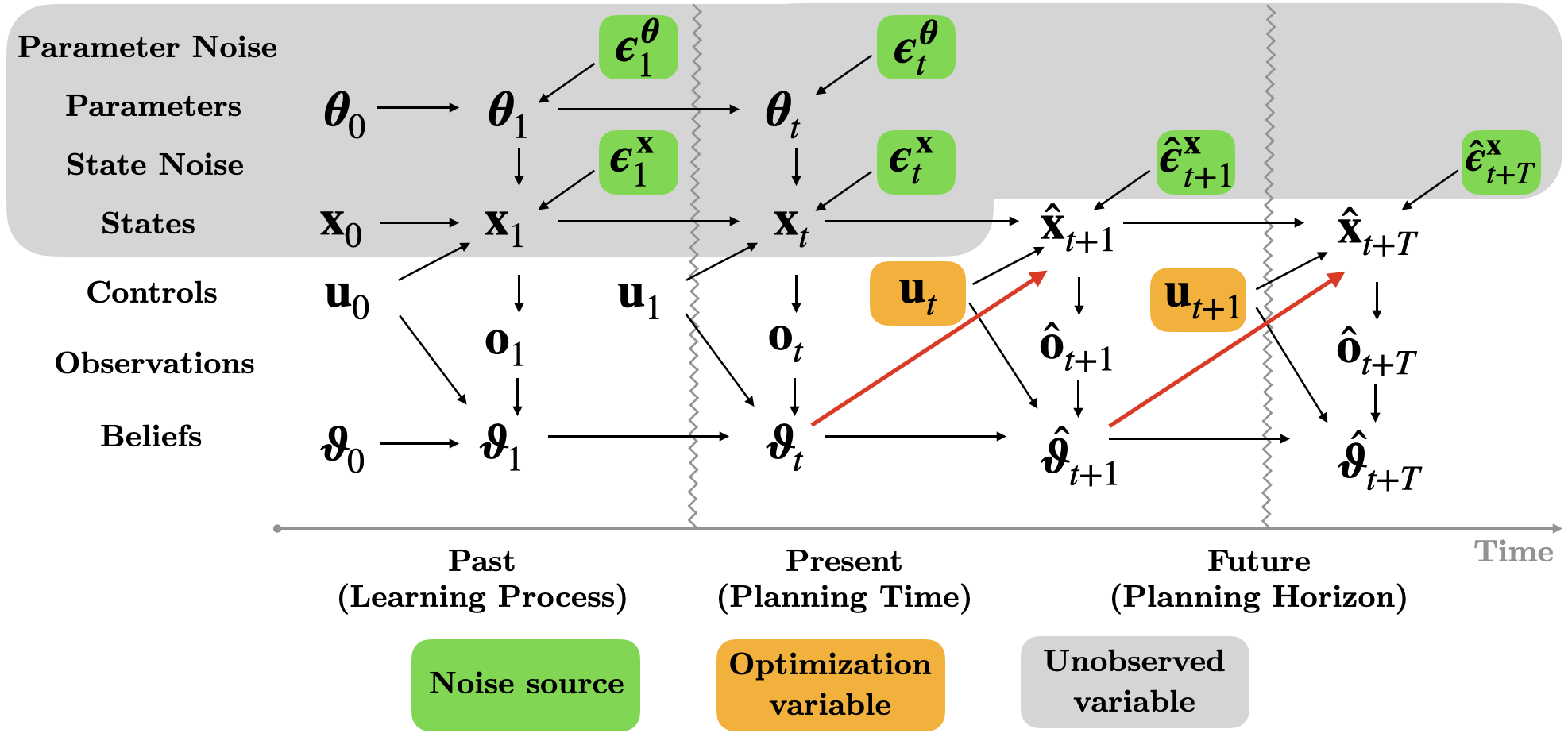}
\caption{Causal graph for states $\state_\timeindex$, observations $\observation_\timeindex$, controls $\control_\timeindex$, parameters $\params_\timeindex$, beliefs $\belief_\timeindex$, and noise $\noise_\timeindex$. 
Planning is done at time $\timestart$ for a planning horizon $\horizon$. 
Zigzag lines denote axis breaks.
A variable with a hat denotes a prediction made during planning.
Red lines emphasize the causal dependencies between states and beliefs during prediction: since the true parameters are not available, the states must be predicted using the latest dynamics parameters belief.}
\label{fig:causal}
\vspace{-10pt}
\end{figure}

\section{Introduction}
An agent operating in an unknown dynamical system must learn it from observations. 
It can do so passively, by simply collecting data as it pursues other cost, or actively, by seeking maximally informative observations. 
In this paper we focus on the problem of active information gathering, which is important because collecting information-dense observations reduces time, energy, and risk during exploration.

The problem of information gathering is hard for two reasons. 
First, it requires evaluating complex nested expectations to compute how future uncertainty changes as a function of planned actions \cite{bar1974dual,klenske2016dual}.
Second, existing formalizations do not cleanly expose the causal dependencies between beliefs, controls, and states. 
Consequently, existing methods rely upon information-gathering costs derived from specific problem structures. 

In this work, we present a unifying framework that cleanly exposes the causal dependencies between parameters, states, controls, observations, beliefs, and predicted beliefs (see \cref{fig:causal} for a visualization) \textit{without} committing to a specific dynamics model type, belief update procedure, belief dynamics model, observation model, or planning procedure \textbf{(Contribution 1, \cref{sec:contribution1})}.
Then we derive a general information-gathering cost based on directed information \cite{massey1990causality}, which we prove captures existing approaches based on mutual information, cf. \cite{sukhija2023optimistic, buisson2020actively,lew2022safe,krause2007nonmyopic,zimmer2018safe,abraham2019active,vantilborgh2025dual,davydov2024first, palafox2025info} for examples, \textbf{(Contribution 2, \cref{sec:contribution2})}.
Equipped with this framework, we derive an explicit connection between information gathering and uncertainty reduction in linearized Bayesian estimation, providing a measure of theoretical justification for mutual information that prior frameworks could not easily expose 
\textbf{(Contribution 3, \cref{sec:contribution3})}.
Finally, we present experiments on single-agent and multi-agent settings, showcasing the practical utility of our theoretical framework \textbf{(Contribution 4, \cref{sec:contribution4})}.
Moving forward, our framework opens the door for new combinations of learning and planning components that prior formalisms could not express.

\section{Related Work}
\label{sec:related}

\textbf{Information Gathering Under Dynamics Uncertainty.}
The problem of selecting actions that lead towards informative observations is typically formalized by defining a distribution (referred to as ``belief'') over the unknown dynamics parameters and maximizing Shannon's mutual information~\cite{shannon1948mathematical,cover1999elements} between parameters and observations, with abundant literature in the areas of Bayesian experimental design~\cite{rainforth2024modern} and active learning~\cite{settles2009active,cohn1994improving,zhan2022comparative}.
These approaches address which observations are informative, and they often assume the ability to directly select or sample observations. 
In contrast, we focus on embodied agents operating in dynamical systems where observations cannot be chosen directly but emerge from the system's state evolution. Therefore, the agents must plan control sequences that steer the system through state space regions that yield informative observations.

Approaches to this problem have modeled the unknown dynamics with Gaussian processes~\cite{sukhija2023optimistic, buisson2020actively, krause2007nonmyopic, krause2008near, zimmer2018safe}, Koopman operators~\cite{abraham2019active}, physics-based models~\cite{vantilborgh2025dual}, and neural networks~\cite{davydov2024first, lew2022safe, palafox2025info}.
Unlike these existing works, which derive bespoke information-gathering costs tailored to specific dynamics models or belief update procedures, our framework requires only the minimal structural assumption of Markov dynamics with additive noise.
We also make no assumptions about the specific dynamics model type or belief updating procedure, e.g., a Bayesian filter.
To this end, we derive a general information-gathering cost based on directed information~\cite{massey1990causality}, and in Theorem~\ref{thm:mi_special_case} prove that mutual information, the cost most commonly used in the literature, is a special case of our formulation.

\textbf{Belief-Space Planning and Dual Control.}
Closely related to our work is the literature on belief-space planning, partially observable Markov decision processes (POMDPs)~\cite{kaelbling1998planning, platt2010belief, van2012motion, van2016motion, patil2015scaling}, and dual control~\cite{feldbaum1960dual,mesbah2018stochastic}.
These approaches primarily address \emph{state} uncertainty, whereas we focus on \emph{dynamics} uncertainty.
The theoretical difficulty of belief-space planning is well-documented: the nonlinearity of Bayesian belief updates makes exact solutions intractable even for simple systems~\cite{klenske2016dual, bar1974dual, wittenmark1995adaptive}.

We reference the literature on belief-space planning because it faces the same computational difficulties as us: uncertainty propagation is challenging, and existing approaches derive information gathering solutions specific to particular dynamics models and belief updating procedures.
Our framework addresses this by explicitly exposing causal dependencies between beliefs, controls, and observations, allowing us to derive a general information-gathering cost that can be instantiated for different dynamics models and belief updating procedures without requiring problem-specific derivations.

\textbf{Multi-Agent Information Gathering.}
Our framework naturally extends to information gathering for a single agent operating within a multi-agent environment.
The problem of uncertainty in multi-agent interactions is explored in the literature on Bayesian games~\cite{harsanyi1967games,osborne2004introduction,shoham2008multiagent}.
Most of this work focuses on optimal decision-making under fixed uncertainty; few works directly examine how to actively \emph{reduce} that uncertainty.
An exception is the literature on deception~\cite{hespanha2000deception, rostobaya2023deception, xu2016modeling, fuchs2012sequential, vurankaya2025deceptive, palafox2024smooth}, which examines how agents strategically manipulate opponents' beliefs in adversarial settings. 
These works focus on controlling what others learn about the ego agent, whereas we focus on how the ego agent can actively learn about others. 
Additionally, we address continuous-state, continuous-action multi-agent dynamical systems with explicit dynamics parameter learning.

In multi-agent robotic applications, information-gathering actions have been explored in autonomous driving~\cite{sadigh2016information, hu2024active, knaup2024active, knaup2025dual, gadginmath2025active, wang2023active, sagheb2025unified}, drone racing~\cite{schwarting2021stochastic}, and pursuit-evasion games~\cite{krusniak2025online}.
These methods build upon belief-space planning for POMDPs~\cite{kaelbling1998planning, platt2010belief, van2012motion, van2016motion, patil2015scaling, gmytrasiewicz2004interactive}, where agents reason about how uncertainty will evolve to select controls that reduce it~\cite{taylor2021active}.
In this work, we consider scenarios where an ego agent actively reduces uncertainty about other agents' parameters. 
Our contribution in the multi-agent setting is demonstrating that our general framework applies directly to multi-agent settings without requiring multi-agent-specific derivations.

\section{A Modular Framework for Sequential Decision-Making Under Dynamics Uncertainty}
\label{sec:contribution1}

We consider agents that sequentially plan, execute, and learn the unknown dynamics parameters.
At each step, the agent solves an optimization problem using a learned dynamics model (model-based control), executes the first action, receives an observation, and immediately updates its model (online learning).

\subsection{Execution}

At execution time $\timestart$, an agent in an environment with state $\state_{\timestart} \in \mathbb{R}^{\dimstate}$ applies control $\control_\timestart \in \mathbb{R}^{\dimcontrol}$, which results in state $\state_{\timestart+1}$.
In this work we assume the environment is a stochastic, Markov dynamical system
\begin{equation}
\label{eq:dynamical_system}
    \state_{\timeindex+1} = \dynamics(\state_\timeindex, \control_\timeindex, \params_{\timeindex+1}) + \statenoise_\timeindex, \quad \forall \timeindex \ge \timestart,
\end{equation}
where $\dynamics$ is a known, possibly nonlinear function, $\params_i \in \mathbb{R}^{\dimparams}$ is a set of (unknown) ground-truth parameters, and $\statenoise_\timeindex$ is noise drawn as follows: $\statenoise_\timeindex \sim \mathcal{N}(\mathbf{0}, \statenoisecov), \statenoisecov \succeq 0$.
We assume the parameters $\params_{\timeindex+1}$ evolve according to the parameter dynamics
\begin{equation}
    \label{eq:process_model}
    \params_{\timeindex+1} = \paramsdynamics(\params_{\timeindex}) + \paramsnoise_{\timeindex}, 
\end{equation}
where $\paramsdynamics$ is a known, possibly nonlinear function independent of state or controls, $\paramsnoise_\timeindex \sim \mathcal{N}(\mathbf{0}, \paramsnoisecov), \paramsnoisecov \succeq 0$.
For brevity, we refer the set of all noise as $\noise = \{\statenoise, \paramsnoise\}$.

The agent does not observe the state $\state_{\timeindex+1}$. 
Instead, it receives an observation $\observation_{\timeindex+1} \in \mathbb{R}^{\dimobs}$ which is a known function of the state $\state_{\timeindex+1}$, i.e., 
\begin{equation}
    \observation_{\timeindex+1} = \observationfn(\state_{\timeindex+1})
\end{equation}

Note how $\state_{\timeindex+1}$ is a function of $\params_{\timeindex+1}$ (and not $\params_\timeindex$) in \cref{eq:dynamical_system} and \cref{fig:causal}.
Using this convention, the observation $\observation_{\timeindex+1}$ informs the belief update about the parameters $\params_{\timeindex+1}$

\subsection{Online Dynamics Parameter Learning}

The agent uses the observation $\observation_\timeindex$ to update a belief $\belief_\timeindex \in \mathbb{R}^{\dimbelief}$  over the parameters, i.e., what it ``thinks'' the ground-truth parameters are.
In this paper, the belief is only over the dynamics parameters (not state). 
$\belief_\timeindex$ may correspond to a point estimate or a distribution over the parameters.
For a point estimate, $\dimbelief = \dimparams$. 
For a distribution, $\dimbelief$ depends on how the distribution is defined. 
E.g., if $\belief_\timeindex$ is a Gaussian then $\dimbelief = \dimparams + \dimparams^2$, where the terms correspond to the dimensions of the mean and the covariance, respectively.

We define the \textit{learning process} $\learningfn$ as a procedure that combines 1) the propagation of the ground-truth parameters and state, and 2) the agent's update to its belief.
Formally, it is denoted as:
\begin{equation}
\label{eq:learning_process}
    \belief_{\timeindex+1} = \learningfn(\belief_\timeindex, \state_\timestart, \params_{\timestart}, \noise_{\timestart:\timestart+i}, \control_{\timestart:\timestart + \timeindex})
\end{equation}
where $\learningfn$ is shown in \cref{alg:general_learning}.
This process is labeled as ``past'' in \cref{fig:causal}.

The learning process is causal (i.e., filtering) such that $\belief_{\timeindex+1}$ only depends on the preceding belief, and the noise and control sequences.
This models the way an agent will sequentially process information as it receives it.

Within the learning process (line 5 of \cref{alg:general_learning}) a \emph{belief updater} $\beliefupdater$ outputs a belief $\belief_{\timeindex+1}$ given the previous belief $\belief_\timeindex$, an observation $\observation_\timeindex$, and additional conditioning variables $\control_\timeindex$, i.e., controls.
Common belief updaters include: gradient descent, where the belief is a point estimate obtained by minimizing prediction error (e.g., mean-squared error); Kalman filters, where the belief is a set of statistics defining a Gaussian distribution over parameters refined via Bayesian updates; linear regression, where the belief is a least-squares point estimate fit to observed transitions; and particle filters, where the belief is a set of weighted parameter samples reweighted and resampled based on observations.
Formally, a belief updater is defined as 
\begin{equation}
    \belief_{\timeindex+1} = \beliefupdater(\belief_{\timestart:\timeindex}, \observation_{\timestart+1:\timeindex+1}, \control_{\timestart:\timeindex}).
\end{equation}
In general, the belief at time $\timeindex+1$ depends on the preceding sequences of beliefs, observations, and controls. 
In practice, most belief updaters are Markovian, i.e., $\belief_{\timeindex+1} = \beliefupdater(\belief_{\timeindex}, \observation_{\timeindex+1}, \control_{\timeindex})$.
The full learning process is shown in \cref{alg:general_learning}.

\begin{algorithm}[ht!]
\SetAlgoLined
\DontPrintSemicolon
\caption{Learning Process $\learningfn$}
\label{alg:general_learning}
\KwIn{$ \belief_\timestart, \state_\timestart, \params_{\timestart}, \noise_{\timestart:\timestart+i}, \control_{\timestart:\timestart + \timeindex}$}
{
    \BlankLine
    \tcp{Propagate state, dynamics, and belief from $\timestart$ to $\timeindex + 1$}
    \For{$j = \timestart \dots \timeindex$}{
        $\params_{j+1} = \paramsdynamics(\params_j) + \paramsnoise_j$
        \tcp*{Nature propagates parameters}
        $\state_{j+1} = \dynamics(\state_j, \control_j, \params_{j+1}) + \statenoise_j$
        \tcp*{Nature propagates state}
        $\observation_{j+1} = \observationfn(\state_{j+1})$
        \tcp*{Generate observation}
        $\belief_{j+1} = \beliefupdater(\belief_{\timestart:j}, \observation_{\timestart+1:j+1}, \control_{\timestart:j})$
        \tcp*{Update belief}
    }
    \BlankLine
    \Return{$\belief_{\timeindex+1}$}
}
\end{algorithm}

\subsection{Planning}
At planning time $\timestart$, the agent computes a sequence of $\horizon$ controls such that it minimizes a cost $\cost$.
Generically, cost can be defined as a function of predicted observations $\observationpredicted_{\timestart+1:\timestart + \horizon}$, controls $\control = \control_{\timeindex:\timeindex+\horizon}$, and predicted beliefs $\beliefpredicted = \beliefpredicted_{\timestart+1:\timestart + \horizon}$, i.e., $\cost(\observationpredicted, \control, \beliefpredicted)$.

The predicted beliefs are given by a \textit{belief dynamics} function $\beliefdynamics$ that models how belief will evolve over time. 
Formally,
\begin{equation}
\label{eq:beliefdynamics}
    \beliefpredicted_{\timeindex+1} = \beliefdynamics(\beliefpredicted_\timeindex, \state_\timestart, \noisepredicted_{\timestart:\timestart+i}, \control_{\timestart:\timestart + \timeindex}),
\end{equation}
where $\noisepredicted_{\timestart:\timestart+i}$ is a sequence of predicted noise samples. 
\cref{fig:causal} sections ``present'' and ``future'' illustrate the causal relationships among the variables in \cref{eq:beliefdynamics}.
We remark that $\beliefdynamics$ may differ from the learning process $\learningfn$. 
This distinction allows us to separate the \textit{actual} belief evolution (observed during execution) from the \textit{predicted} belief evolution (used during planning). 

Finally, the agent's optimization problem is given by
\begin{subequations}
\label{eq:optimization_problem}
    \begin{align}
        \min_\control \, \cost(\observationpredicted, \control, \beliefpredicted)&\\
        \textrm{s.t.} \, \beliefpredicted_{\timeindex+1} &= \beliefdynamics(\beliefpredicted_\timeindex, \state_\timestart, \noisepredicted_{\timestart:\timestart+i}, \control_{\timestart:\timestart + \timeindex})\\
        \statepredicted_{\timeindex+1} &= \dynamics(\statepredicted_{\timeindex}, \control_\timeindex, \beliefpredicted_{\timeindex+1}) + \hat{\statenoise}_\timeindex\\
        \observationpredicted_{\timeindex+1} &= \observationfn(\statepredicted_{\timeindex+1}),
    \end{align}
\end{subequations}
where the initial state $\statepredicted_\timestart = \state_\timestart$, and belief $\beliefpredicted_\timestart = \belief_\timestart$ are assumed to be known.
We refer to the \textit{planner} as the algorithm used to solve problem \eqref{eq:optimization_problem}.

\contribution[A modular framework for sequential decision-making under dynamics uncertainty]{
We present a framework that cleanly exposes the causal dependencies between parameter beliefs, controls, and observations.
This framework unifies existing work on decision making under dynamics uncertainty and allows us to compose dynamics models (e.g., Gaussian processes, physics-based models, neural networks, etc.), belief updaters (e.g., Kalman filters, particle filters, online gradient descent), belief dynamics, observation models, and planners, while maintaining the same causal structure.
This invariance allows us to derive a general information-gathering cost that is agnostic to these choices (\cref{sec:contribution2}), unlike existing information-based costs which explicitly depend on the chosen dynamics model and update procedure \cite{sukhija2023optimistic, buisson2020actively, krause2007nonmyopic, krause2008near,zimmer2018safe,abraham2019active,vantilborgh2025dual,davydov2024first, lew2022safe, palafox2025info}.
Finally, the modularity of the framework also allows us to make an explicit connection to learning in (\cref{sec:contribution3}).}

\section{A General Information-Gathering Cost for Dynamics Uncertainty}
\label{sec:contribution2}

In this section, we derive a general information-gathering cost based on Massey's directed information \cite{massey1990causality}, which captures the causal structure of sequential decision-making exposed in the framework of \Cref{sec:contribution1}. 
This generality allows practitioners to
model information gain across different domains without re-deriving cost functions depending on the chosen dynamics, e.g., mutual information costs for Gaussian process dynamics \cite{sukhija2023optimistic, buisson2020actively, krause2007nonmyopic, krause2008near, zimmer2018safe}.
We conclude this section with \Cref{thm:mi_special_case}, in which we prove that under certain assumptions, this cost reduces to the popular mutual information cost used in the information-gathering literature. 

\subsubsection{Mutual Information}
Mutual information quantifies the information obtained about  random variable $X$ by observing random variable $Y$ \cite[Def. 8.9.1]{cover1999elements}.
Formally, it is defined as
\begin{subequations}
    \begin{align}
        \mutualinfo(X;Y) &\triangleq \entropy(X) - \entropy(X\mid Y)\\
        &= \mathbb{E}_{X\sim p(x)}\left[-\log p(X)\right] - \mathbb{E}_{(X,Y)\sim p(x,y)}\left[-\log p(X|Y)\right] \\
        &= -\int p(x)\log p(x)\,dx + \int p(x,y)\log p(x|y)\,dx\,dy
        \label{eq:mutual_information}
    \end{align}
\end{subequations}
where $\entropy$ is the differential entropy of a random variable \cite[Def. 8.1]{cover1999elements}.

One could encode information-gathering behavior by maximizing mutual information between the sequences $\beliefpredicted$ and $\observationpredicted$. 
Unfortunately, this would fail to account for their causal dependencies, i.e., in sequential decision making, future beliefs do not influence past observations (see \cref{fig:causal}). 
Therefore, in the next section we derive an information-gathering cost based on directed information \cite{massey1990causality}, a metric that accounts for the causal relationship between two random sequences.

\subsubsection{Directed Information}

Following prior work on directed information~\cite{kramer1998directed,massey1990causality},
we begin by defining the \textit{causally} conditioned probability of a random variable.
Given two sequences of random variables $Y = Y_{\timestart+1:\timestart+\horizon}$ and $X = X_{\timestart:\timestart+\horizon-1}$, 
the causally-conditional probability of $Y$ given $X$ is the product of the probability of each element $Y_i$ conditioned only on the history available at each time step:
\begin{equation}
    p(Y \parallel X) \triangleq \prod_{\timeindex=\timestart}^{\timestart+\horizon-1} p(Y_{\timeindex+1} \mid Y_{\timestart+1:\timeindex}, X_{\timestart:\timeindex}),
\label{eq:causal_cond_prob}
\end{equation}
where we have slightly abused notation such that $Y_{\timestart+1:\timestart} = Y_{\timestart+1}$. 
The \textit{causal} entropy \cite[Def. 3.1]{kramer1998directed} is given by applying the definition of entropy to the distribution in \eqref{eq:causal_cond_prob} as follows:
\begin{subequations}
\label{eq:causal_entropy}
    \begin{align}
    \entropy(Y \parallel X) &\triangleq -\mathbb{E}_{Y \sim P(Y\parallel X)}\left[ \log p(Y \parallel X) \right] \\
    &= -\mathbb{E}_{Y \sim P(Y\parallel X)}\left[ \log \left( \prod_{\timeindex=\timestart}^{\timestart+\horizon-1} p(Y_{\timeindex+1} \mid Y_{\timestart+1:\timeindex}, X_{\timestart:\timeindex}) \right) \right] \\
    &= \sum_{\timeindex=\timestart}^{\timestart+\horizon-1} \entropy(Y_{\timeindex+1} \mid Y_{\timestart+1:\timeindex}, X_{\timestart:\timeindex}).
    \end{align}
\end{subequations}
Intuitively, causal entropy measures the accumulated uncertainty as the random variables $X, Y$ are observed sequentially.

The directed information  between sequences $X$ and $Y$ is almost identical to the definition of mutual information in \eqref{eq:mutual_information} but uses the causally-conditional probability of the sequences \cite[Def. 3.5]{kramer1998directed}, i.e.,
\begin{equation}
\label{eq:directed_information}
    \directedinfo(X \rightarrow Y) \triangleq \entropy(Y) - \entropy(Y \parallel X)
\end{equation}
Intuitively, directed information is the total information  about $Y$ when the agent updates its uncertainty over $Y_\timeindex$ using only the histories $Y_{\timestart:\timeindex-1}$ and $X_{\timestart:\timeindex-1}$. 

We often need to quantify information flow conditioned on a third sequence $Z$. 
For example, in our setting $Z$ will be the control sequence which causally influences both beliefs and observations.
In this case, it is useful to use causally \textit{conditional} directed information, which encodes the reduction in the causal entropy of $Y$ given $Z$, when $X$ is also causally observed \cite[Def.~3.7]{kramer1998directed}.
Formally, it is given by:
\begin{equation}
    \mutualinfo(X \rightarrow Y \parallel Z) \triangleq \entropy(Y \parallel Z) - \entropy(Y \parallel X, Z).
    \label{eq:causal_cond_directed_info}
\end{equation}
This metric quantifies the information flow from $X$ to $Y$ that is not explained by the causal history of $Z$.

\subsubsection{The Directed Information Cost}

In our setting, an information-gathering agent maximizes directed information between predicted beliefs $\beliefpredicted$ and predicted observations $\observationpredicted$, causally conditioned on $\control$.
Formally, the cost is
\begin{equation}
    \cost_{\text{info}}(\observationpredicted, \control, \beliefpredicted) = -\directedinfo(\beliefpredicted \rightarrow \observationpredicted \parallel \control) = -\left(\entropy(\observationpredicted \parallel \control)-\entropy(\observationpredicted \parallel \control,\beliefpredicted)\right)
    \label{eq:objective_directed_info}
\end{equation}
Computing the term $\entropy(\observationpredicted \parallel \control)$ requires carefully following the causal dependencies:
the predicted observation $\observationpredicted_{\timeindex+1}$ depends on the predicted state $\statepredicted_{\timeindex+1}$ which depends on the belief $\beliefpredicted_\timeindex$\footnote{The dependence between $\statepredicted_{\timeindex+1}$ and $\beliefpredicted_\timeindex$ is shown by the red lines in \cref{fig:causal}}. 
Finally, the belief $\beliefpredicted_\timeindex$ depends on $\observationpredicted_{\timestart:\timeindex}$ and $\control_{\timestart:\timeindex}$.
Therefore, $\observationpredicted_{\timeindex+1}$ is conditioned on the history $\observationpredicted_{\timestart:\timeindex}$ and $\control_{\timestart:\timeindex}$:
\begin{equation}
    \entropy(\observationpredicted \parallel \control) = \sum_{\timeindex=\timestart}^{\timestart+\horizon-1} \entropy(\observationpredicted_{\timeindex+1} \mid \observationpredicted_{\timestart:\timeindex}, \control_{\timestart:\timeindex}).
    \label{eq:obs_causal_entropy_marg}
\end{equation}
On the other hand, computing $\entropy(\observationpredicted \parallel \control, \beliefpredicted)$ is easier since it can be expressed as a sum of entropies by applying \eqref{eq:causal_entropy}, and then simplified with the Markov property of the system \eqref{eq:dynamical_system}:
\begin{equation}
        \entropy(\observationpredicted \parallel \control, \beliefpredicted) = \sum_{\timeindex=\timestart}^{\timestart+\horizon-1} \entropy(\observationpredicted_{\timeindex+1} \mid \observationpredicted_{
        \timestart: \timeindex}, \control_{
        \timestart: \timeindex}, \beliefpredicted_{
        \timestart: \timeindex})
        = \sum_{\timeindex=\timestart}^{\timestart+\horizon-1} \entropy(\observationpredicted_{\timeindex+1} \mid \observationpredicted_\timeindex, \control_\timeindex, \beliefpredicted_\timeindex).
    \label{eq:obs_causal_entropy}
\end{equation}
Then, substituting into \eqref{eq:objective_directed_info} yields the \textbf{directed information cost}:
\begin{equation}
    \boxed{\cost_{\text{info}}(\observationpredicted, \control, \beliefpredicted) = -\sum_{\timeindex=\timestart}^{\timestart+\horizon-1} \left[ \entropy(\observationpredicted_{\timeindex+1} \mid \observationpredicted_{\timestart:\timeindex},\control_{\timestart:\timeindex}) - \entropy(\observationpredicted_{\timeindex+1} \mid \observationpredicted_\timeindex, \control_\timeindex, \beliefpredicted_\timeindex)\right],}
    \label{eq:final_directed_cost}
\end{equation}
 where 
\begin{equation}
\entropy(\observationpredicted_{\timeindex+1}\mid \observationpredicted_{\timestart:\timeindex}, \control_{\timestart:\timeindex}) = -\int p(\observationpredicted_{\timeindex+1} \mid \observationpredicted_{\timestart:\timeindex}, \control_{\timestart:\timeindex})\log p(\observationpredicted_{\timeindex+1} \mid \observationpredicted_{\timestart:\timeindex}, \control_{\timestart:\timeindex}) d \observationpredicted_{\timeindex + 1},
\end{equation}
\begin{equation}
\label{eq:obs_marginal}
    p(\observationpredicted_{\timeindex+1} | \observationpredicted_{\timestart:\timeindex}, \control_{\timestart:\timeindex}) = \int p(\observationpredicted_{\timeindex+1} | \observationpredicted_\timeindex, \control_\timeindex, \beliefpredicted_{\timeindex}) \underbrace{p(\beliefpredicted_{\timeindex}|\observationpredicted_{\timestart:\timeindex}, \control_{\timestart:\timeindex})}_{\textrm{belief dynamics }}d\beliefpredicted_{\timeindex}.
\end{equation}
Note how \eqref{eq:final_directed_cost} applies to arbitrary dynamics models, belief update procedures, and observation models.

\subsubsection{Connection to Mutual Information Cost}

We now prove that under certain assumptions, our directed information cost reduces to the mutual information cost in the information gathering literature \cite{sukhija2023optimistic, buisson2020actively,lew2022safe,krause2007nonmyopic,zimmer2018safe,abraham2019active,vantilborgh2025dual,davydov2024first, palafox2025info}.

\begin{theorem}[Mutual Information Cost as a Special Case of the Directed Information Cost]
\label{thm:mi_special_case}
Under the following assumptions:
\begin{enumerate}
    \item \textbf{Full observability:} $\observation_{\timeindex} = \state_{\timeindex}$
    \item \textbf{Gaussian belief:} $\beliefpredicted_{\timeindex}  = \{\beliefmean_\timeindex, \beliefcov_\timeindex\} \implies p(\beliefpredicted_{j} \mid \statepredicted_{\timestart:\timeindex}, \control_{\timestart:\timeindex}) = \mathcal{N}(\beliefpredicted_{j}; \beliefmean_\timeindex, \beliefcov_\timeindex )$
    \item \textbf{Static belief dynamics:} $\beliefpredicted_{\timeindex} = \belief_\timestart, \forall \timeindex \ge \timestart$
    \item \textbf{Linearized state dynamics:} $\dynamics(\statepredicted_\timeindex, \control_\timeindex, \beliefpredicted_{\timeindex}) \approx \dynamics(\statepredicted_\timeindex, \control_\timeindex, \beliefmean_\timeindex) + \dynjac_\timeindex (\beliefpredicted_{\timeindex} - \beliefmean_\timeindex)$ where $\dynjac_\timeindex = \frac{\partial \dynamics}{\partial \params}(\statepredicted_\timeindex, \control_\timeindex, \beliefmean_\timeindex)$.
\end{enumerate}
the directed information cost \eqref{eq:final_directed_cost} reduces to \cref{eq:mutual_info_cost}, which is the sum of conditional mutual informations between the starting parameter belief $\belief_\timestart$ and the states $\statepredicted_{\timeindex}, \forall \timestart \leq \timeindex \leq \timestart + \horizon$, i.e., $\sum_{\timeindex = \timestart}^{\timestart + \horizon - 1}\mutualinfo(\belief_\timestart; \statepredicted_{\timeindex+1} \mid \statepredicted_{\timeindex}, \control_{\timeindex})$. 
\begin{equation}
    \cost_{\text{info}}(\statepredicted, \control, \beliefpredicted) = - \frac{1}{2} \sum_{\timeindex=\timestart}^{\timestart+\horizon-1} \log \left(\frac{\det (\innovationcovpredicted_\timeindex)}{\det (\statenoisecov)}\right),
    \label{eq:mutual_info_cost}
\end{equation}
where $\innovationcovpredicted_\timeindex = \dynjac_\timeindex \beliefcov_\timestart \dynjac_\timeindex^\top + \statenoisecov$.
\end{theorem}

\begin{proof}
Under Assumption 1,  the state is directly observed, so $\observationpredicted_{\timeindex+1} = \statepredicted_{\timeindex+1}$ and the directed information cost \eqref{eq:final_directed_cost} becomes
\begin{equation}
    \cost_{\text{info}}(\statepredicted, \control, \beliefpredicted) = -\sum_{\timeindex=\timestart}^{\timestart+\horizon-1} \left[ \entropy(\statepredicted_{\timeindex+1} \mid \statepredicted_{\timestart:\timeindex},\control_{\timestart:\timeindex}) - \entropy(\statepredicted_{\timeindex+1} \mid \statepredicted_\timeindex, \control_\timeindex, \beliefpredicted_\timeindex)\right].
    \label{eq:cost_under_a1}
\end{equation}
Under Assumption 3, the belief is static and at each timestep $\timeindex > \timestart$ is therefore independent of history $(\statepredicted_{\timestart:\timeindex-1}, \control_{\timestart:\timeindex-1})$. 
This implies that given $(\statepredicted_\timeindex, \control_\timeindex, \beliefpredicted_\timeindex)$, the next state $\statepredicted_{\timeindex+1} = \dynamics(\statepredicted_\timeindex, \control_\timeindex, \beliefpredicted_\timeindex) + \statenoise_\timeindex$ is also independent of the history.
Therefore, the first term in \cref{eq:cost_under_a1} simplifies as 
\begin{equation}
    \entropy(\statepredicted_{\timeindex+1} \mid \statepredicted_{\timestart:\timeindex}, \control_{\timestart:\timeindex}) = \entropy(\statepredicted_{\timeindex+1} \mid \statepredicted_\timeindex, \control_\timeindex, \belief_\timestart).
    \label{eq:markov_requirement}
\end{equation}
Substituting into \cref{eq:cost_under_a1} results in 
\begin{subequations}
    \begin{align}
        \cost_{\text{info}}(\statepredicted, \control, \beliefpredicted) &= -\sum_{\timeindex=\timestart}^{\timestart+\horizon-1} \left[ \entropy(\statepredicted_{\timeindex+1} \mid \statepredicted_{\timestart},\control_{\timestart}) - \entropy(\statepredicted_{\timeindex+1} \mid \statepredicted_\timeindex, \control_\timeindex, \belief_\timestart)\right] \label{eq:cost_as_mi}\\ &= -\sum_{\timeindex=\timestart}^{\timestart+\horizon-1} \mutualinfo(\belief_\timestart; \statepredicted_{\timeindex+1} \mid \statepredicted_{\timeindex}, \control_{\timeindex}).
    \end{align}
\end{subequations}
where $\mutualinfo(\belief_\timestart; \statepredicted_{\timeindex+1} \mid \statepredicted_{\timeindex}, \control_{\timeindex})$ is the conditional mutual information between the initial belief $\belief_\timestart$ and predicted state $\statepredicted_{\timeindex+1}$, conditioned on the previous predicted state and control, $\statepredicted_{\timeindex} \control_{\timeindex}$, respectively.
This shows that, under Assumptions 1 and 3, our directed information cost \cref{eq:final_directed_cost} reduces to a mutual information cost.
We now evaluate each entropy term and simplify further. 

For the second term in \eqref{eq:cost_as_mi}, given $\statepredicted_\timeindex$, $\control_\timeindex$, and $\belief_\timestart$, the only randomness in $\statepredicted_{\timeindex+1}$ is the Gaussian state noise $\statenoise_\timeindex \sim \mathcal{N}(\mathbf{0}, \statenoisecov)$.
Thus, $\entropy(\statepredicted_{\timeindex+1} \mid \statepredicted_\timeindex, \control_\timeindex, \belief_\timestart) = \frac{1}{2}\log\det(2\pi e \, \statenoisecov)$, cf. \cite[Theorem~8.4.1]{cover1999elements}.

To evaluate the first term in \eqref{eq:cost_as_mi}, we must marginalize over the belief.
Under Assumption 2, the belief is Gaussian: $p(\beliefpredicted_\timeindex) = \mathcal{N}(\beliefpredicted_\timeindex;\beliefmean_\timestart, \beliefcov_\timestart)$.
Under Assumption 4, $\statepredicted_{\timeindex+1}$ is an affine function of $\beliefpredicted_\timeindex$ plus Gaussian noise:
\begin{equation}
    \statepredicted_{\timeindex+1}(\beliefpredicted_{\timeindex}) = \dynamics(\statepredicted_\timeindex, \control_\timeindex, \beliefmean_\timestart) + \dynjac_\timeindex (\beliefpredicted_\timeindex - \beliefmean_\timestart) + \statenoise_\timeindex.
\end{equation}
Since affine transformations of Gaussian random variables remain Gaussian, marginalizing over $\beliefpredicted_\timeindex$ yields:
\begin{equation}
    p(\statepredicted_{\timeindex+1} \mid \statepredicted_\timeindex, \control_\timeindex) = \int p(\statepredicted_{\timeindex+1} \mid \statepredicted_\timeindex, \control_\timeindex, \beliefpredicted_\timeindex) p(\beliefpredicted_\timeindex) d\beliefpredicted_\timeindex = \mathcal{N}(\dynamics(\statepredicted_\timeindex, \control_\timeindex, \beliefmean_\timestart), \innovationcovpredicted_\timeindex),
\end{equation}
where $\innovationcovpredicted_\timeindex = \dynjac_\timeindex \beliefcov_\timestart \dynjac_\timeindex^\top + \statenoisecov$.
Thus, $\entropy(\statepredicted_{\timeindex+1} \mid \statepredicted_\timeindex, \control_\timeindex) = \frac{1}{2}\log\det(2\pi e \, \innovationcovpredicted_\timeindex)$.
Substituting both entropy terms yields the desired expression, as follows:
\begin{subequations}
    \begin{align}
        \cost_{\text{info}}(\statepredicted, \control, \beliefpredicted) &= -\sum_{\timeindex=\timestart}^{\timestart+\horizon-1} \left[ \frac{1}{2}\log\det(2\pi e \, \innovationcovpredicted_\timeindex) - \frac{1}{2}\log\det(2\pi e \, \statenoisecov) \right]\\
        &= -\frac{1}{2}\sum_{\timeindex=\timestart}^{\timestart+\horizon-1} \log\left( \frac{\det(\innovationcovpredicted_\timeindex)}{\det(\statenoisecov)} \right).
    \end{align}
\end{subequations}
\qed
\end{proof}

\contribution[A general information-gathering cost based on directed information.]{
Using our framework, we derive a novel information-gathering cost that captures the causal structure of sequential decision-making under dynamics uncertainty.
Unlike existing approaches, which derive bespoke costs tailored to specific dynamics models or belief update procedures, our cost requires only minimal structural assumptions (Markov dynamics with additive noise) and makes no assumptions about the dynamics model, belief updater, belief dynamics, observations model, or planner.
In \Cref{thm:mi_special_case}, we prove that the popular mutual information cost is a special case (cf. \cite{sukhija2023optimistic, buisson2020actively,lew2022safe,krause2007nonmyopic,zimmer2018safe,abraham2019active,vantilborgh2025dual,davydov2024first, palafox2025info} for examples).}

\section{An Explicit Connection Between Information Gathering and Dynamics Parameter Learning}
\label{sec:contribution3}

\begin{wrapfigure}{r}{0.55\textwidth}
\vspace{-20pt}
\begin{algorithm}[H]
{\small
\SetAlgoLined
\DontPrintSemicolon
\caption{Extended Kalman Filter Belief Updater $\beliefupdater$}
\label{alg:ekf}
\KwIn{$\belief_\timeindex = \{\beliefmean_\timeindex, \beliefcov_\timeindex \}$, $\observation_{\timeindex+1}$, and $\control_\timeindex$.}
{
    \BlankLine
    \tcp{1. Prediction Step}
    $\beliefmean_{\timeindex+1|\timeindex} = \paramsdynamics(\beliefmean_{\timeindex})$\\
    $\paramsprocjac_\timeindex = \frac{\partial \paramsdynamics}{\partial \params}(\beliefmean_\timeindex)$\\
    $\beliefcov_{\timeindex+1|\timeindex} = \paramsprocjac_\timeindex \beliefcov_\timeindex \paramsprocjac_\timeindex^\top + \paramsnoisecov$\\
    $\statepredicted_{\timeindex+1|\timeindex} = \dynamics(\state_{\timeindex}, \control_\timeindex, \beliefmean_{\timeindex+1|\timeindex})$\\
    $\observationpredicted_{\timeindex+1|\timeindex} = \observationfn(\statepredicted_{\timeindex+1|\timeindex})$
    \BlankLine
    \tcp{2. Correction Step}
    $\dynjac_{\timeindex} = \frac{\partial \dynamics}{\partial \params}(\state_{\timeindex}, \control_{\timeindex}, \beliefmean_{\timeindex+1|\timeindex})$\\
    $\obsjac_{\timeindex} = \frac{\partial \observationfn}{\partial \state}(\statepredicted_{\timeindex+1|\timeindex})$\\
    $\innovationcov_\timeindex = \obsjac_{\timeindex} \dynjac_\timeindex \beliefcov_{\timeindex+1|\timeindex} \dynjac_\timeindex^\top \obsjac_{\timeindex}^\top + \obsjac_{\timeindex} \statenoisecov \obsjac_{\timeindex}^\top$\\
    $\kalmangain_\timeindex = \beliefcov_{\timeindex+1|\timeindex}\dynjac_\timeindex^\top \obsjac_{\timeindex}^\top \innovationcov_\timeindex^{-1}$
    \\$\innovation_{\timeindex} = \observation_{\timeindex+1} - \observationpredicted_{\timeindex+1|\timeindex}$\\
    $\beliefmean_{\timeindex+1} = \beliefmean_{\timeindex+1|\timeindex} + \kalmangain_\timeindex\innovation_{\timeindex}$\\
    $\beliefcov_{\timeindex+1} = (\mathbf{I} - \kalmangain_\timeindex \obsjac_{\timeindex} \dynjac_\timeindex)\beliefcov_{\timeindex+1|\timeindex}$
    \BlankLine
    \Return{$\belief_{\timeindex+1} = \{\beliefmean_{\timeindex+1}, \beliefcov_{\timeindex+1} \}$}
}
}
\end{algorithm}
\vspace{-20pt}
\end{wrapfigure}

In this section, we show that minimizing the mutual information cost of \Cref{thm:mi_special_case} directly maximizes the information gain in the dynamics parameter learning process when using a linearized Bayesian belief updater.
This analysis is enabled by the fact that our framework cleanly separates the learning process from the belief dynamics used during planning, and explicitly exposes causal dependencies between predicted beliefs, observations, and controls. 
This modularity comes at \emph{no performance cost}; rather, it  makes explicit the design choices (belief updater and dynamics) that are generally left implicit in the literature.

A common choice for the belief updater is the Extended Kalman Filter (EKF) shown in \cref{alg:ekf}, cf. \cite{thrun2002probabilistic,ljung1979asymptotic, goodwin2014adaptive}. 
The EKF defines the parameter belief as the mean $\beliefmean_\timeindex$ and covariance $\beliefcov_\timeindex$ of a Gaussian distribution, i.e., $\belief_\timeindex = \{\beliefmean_\timeindex, \beliefcov_\timeindex\}$ (as in Assumption 2 in \Cref{thm:mi_special_case}).
The EKF generalizes \cite{howell2025optimization,ollivier2018online} the bespoke Newton-like updates used in information-gathering literature, cf. \cite{davydov2024first,lew2022safe,harrison2024variational,harrison2020meta,hu2024active}. 


To see the explicit connection between the EKF and the mutual information cost, we rewrite the EKF's covariance update in (line 12, \cref{alg:ekf}) in information form:
\begin{equation}
    \beliefcov_{\timeindex+1}^{-1} = \beliefcov_{\timeindex+1|\timeindex}^{-1} + \dynjac_\timeindex^\top (\statenoisecov)^{-1} \dynjac_\timeindex. 
    \label{eq:}
\end{equation}
To compare with the mutual information cost, we express the EKF's information in the same units, i.e., $\log \det(\beliefcov_{\timeindex+1}^{-1})$.
Under Assumption 3 in \Cref{thm:mi_special_case}, this simplifies to $\log \det(\beliefcov_\timestart^{-1} + \dynjac_\timeindex^\top (\statenoisecov)^{-1} \dynjac_\timeindex)$. 
Both $\log \det(\beliefcov_{\timeindex+1}^{-1})$ and $\log \det(\innovationcovpredicted_\timeindex)$ from the mutual information cost \eqref{eq:mutual_info_cost} increase with the singular values of $\dynjac_\timeindex$.
Thus, minimizing the mutual information cost simultaneously maximizes the EKF information gain, providing a direct link between the planning cost and learning performance.

\contribution[An explicit connection between information gathering and learning.]{
We show that minimizing a mutual information-based cost is equivalent to maximizing information gain in the belief update for the case of linearized Bayesian estimation. 
To our knowledge, this is the first time this connection has been explicitly made, and the framework in \Cref{sec:contribution1,sec:contribution2} makes the analysis very straightforward.}

\section{Experiments}
\label{sec:contribution4}

We showcase the practical utility of our framework through several experiments in both single-robot and multi-robot systems. 
These experiments test two hypotheses about the benefits of active information gathering.
These hypotheses are not controversial in the literature, and are widely supported by existing work in information-aware planning, cf. \cite{davydov2024first,lew2022safe,harrison2024variational,harrison2020meta,hu2024active}.
Our objective here is to illustrate that an implementation of our proposed framework can replicate these patterns, and thus that it holds practical as well as theoretical value.

\paragraph{Hypothesis 1: Information gathering accelerates parameter learning and reduces uncertainty compared with random actions or passive learning, where the agent does not actively seek information-rich observations.}
We measure parameter estimation error $\|\bar{\vartheta}_i - \theta_i\|$ and covariance trace $\text{tr}(\Sigma_i)$ over time.

\paragraph{Hypothesis 2: Information gathering produces learned models that generalize better to unseen data.}
We compare model predictions against held-out state transitions the agent has never observed, reporting both single-step prediction error $\|x_{i+1} - f(x_i, u_i, \bar{\vartheta}_{i+1})\|$ and total prediction error computed by autoregressively predicting entire trajectories.

\subsubsection{Implementation Details} 

We use the EKF from \cref{alg:ekf} as our belief updater, the mutual information cost from \Cref{thm:mi_special_case} (which assumes full observability), and a sampling-based cross-entropy method as our planner \cite{pinneri2021sample}.
In all experiments, we perform online learning by running the EKF after every observation, following \cref{alg:general_learning}.
The complete planning cost combines the information-gathering cost with a task-specific cost.
Formally,
\begin{equation}
    \cost(\statepredicted, \control, \beliefpredicted) = \cost_{\text{task}}(\statepredicted, \control) + \lambda \cost_{\text{info}}(\statepredicted, \control, \beliefpredicted)
\end{equation}
where $\cost_{\text{info}}$ is the mutual information cost from \Cref{thm:mi_special_case}, and $\lambda \in \mathbb{R}$ is the independent variable varied to test the effects of information gathering. 
$\lambda = 0$ corresponds to the case of passive learning.
Full implementation details and hyperparameters are publicly available in the code repository.\footnote{\url{https://github.com/fernandopalafox/max}}

In all cases, the agent operates in a dynamical system that matches its model structure, and mismatches between observations and predictions arise solely from the difference between the ground-truth parameters and the parameter belief.
For multi-agent experiments, we partition the state into components corresponding to the ego and any opponents, planning from the perspective of agent 1 (ego) refining its model of other agents. 
For example, in a two-agent system the state is $\state_\timeindex^\top = \begin{bmatrix}
        \state_\timeindex^{1\top} & 
        \state_\timeindex^{2^\top}
    \end{bmatrix}$
where superscripts index by agent.
Next we describe the experiments

\begin{figure}[!ht]
\centering
\begin{minipage}{\linewidth}
\begin{minipage}{0.48\linewidth}
\centering
\includegraphics[width=\linewidth]{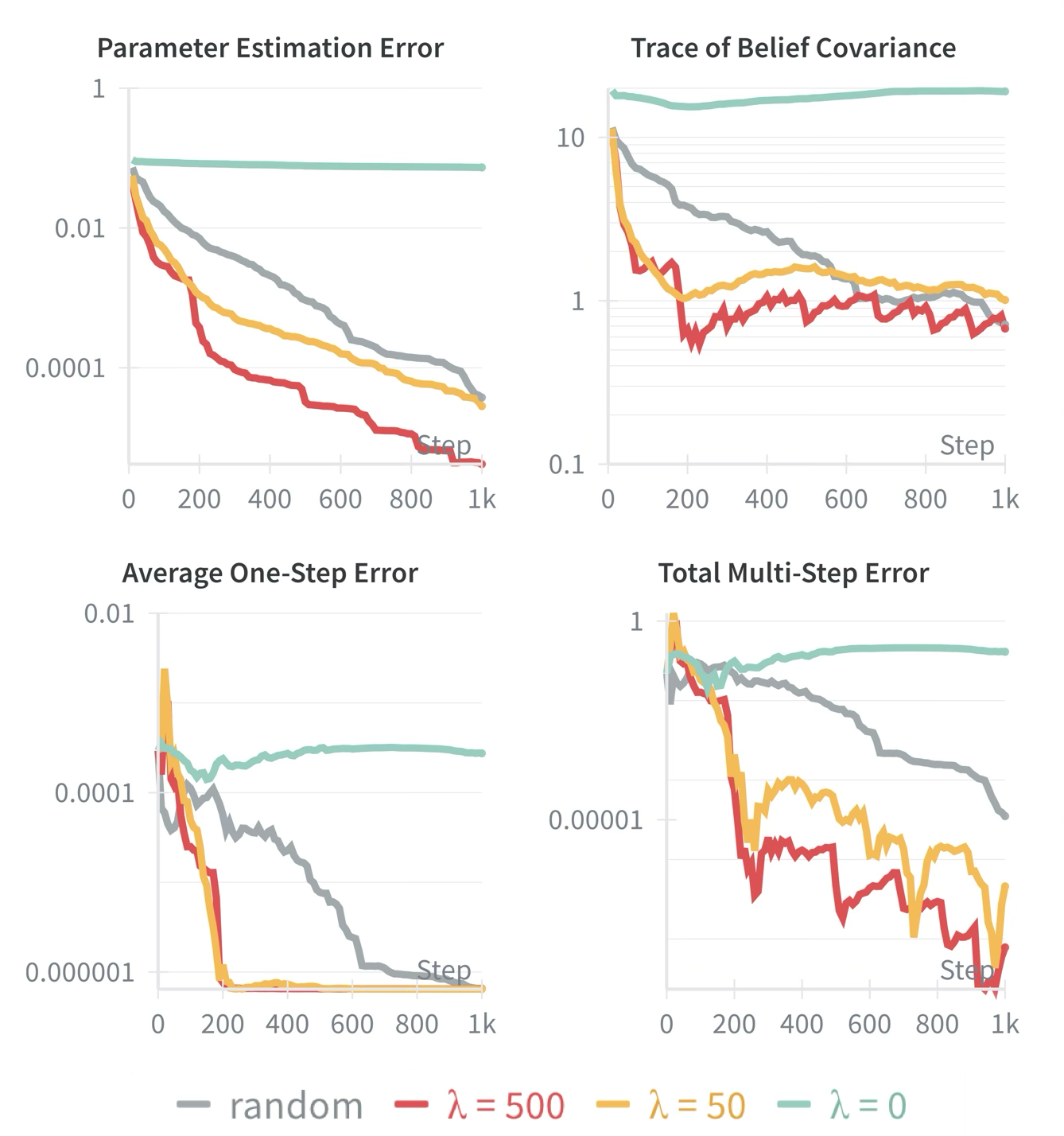}
\caption{Double integrator}
\label{fig:results_linear}
\end{minipage}
\hfill
\begin{minipage}{0.48\linewidth}
\centering
\includegraphics[width=\linewidth]{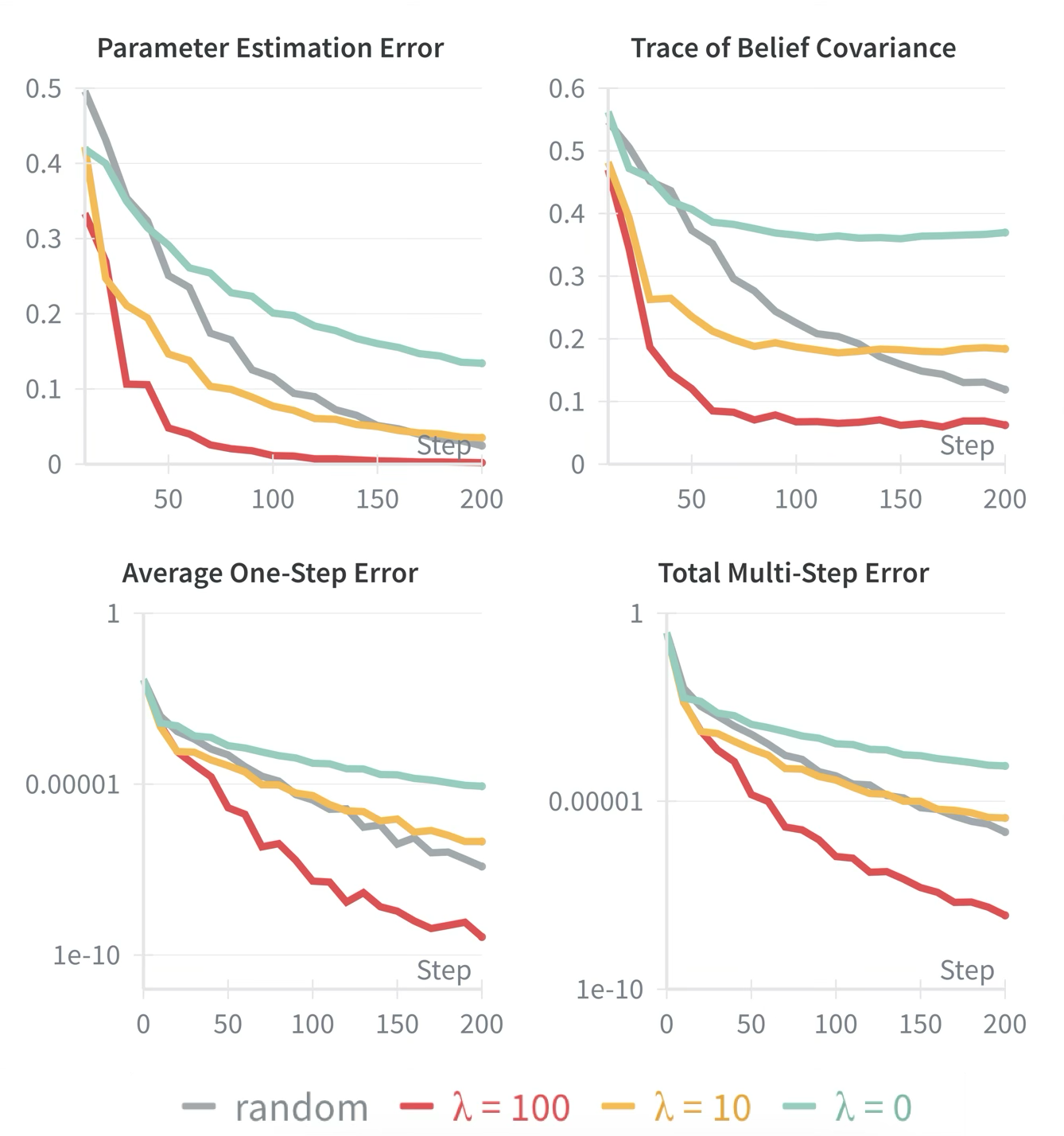}
\caption{Damped pendulum}
\label{fig:results_pendulum}
\end{minipage}
\end{minipage}

\vspace{1pt} 

\begin{minipage}{\linewidth}
\begin{minipage}{0.48\linewidth}
\centering
\includegraphics[width=\linewidth]{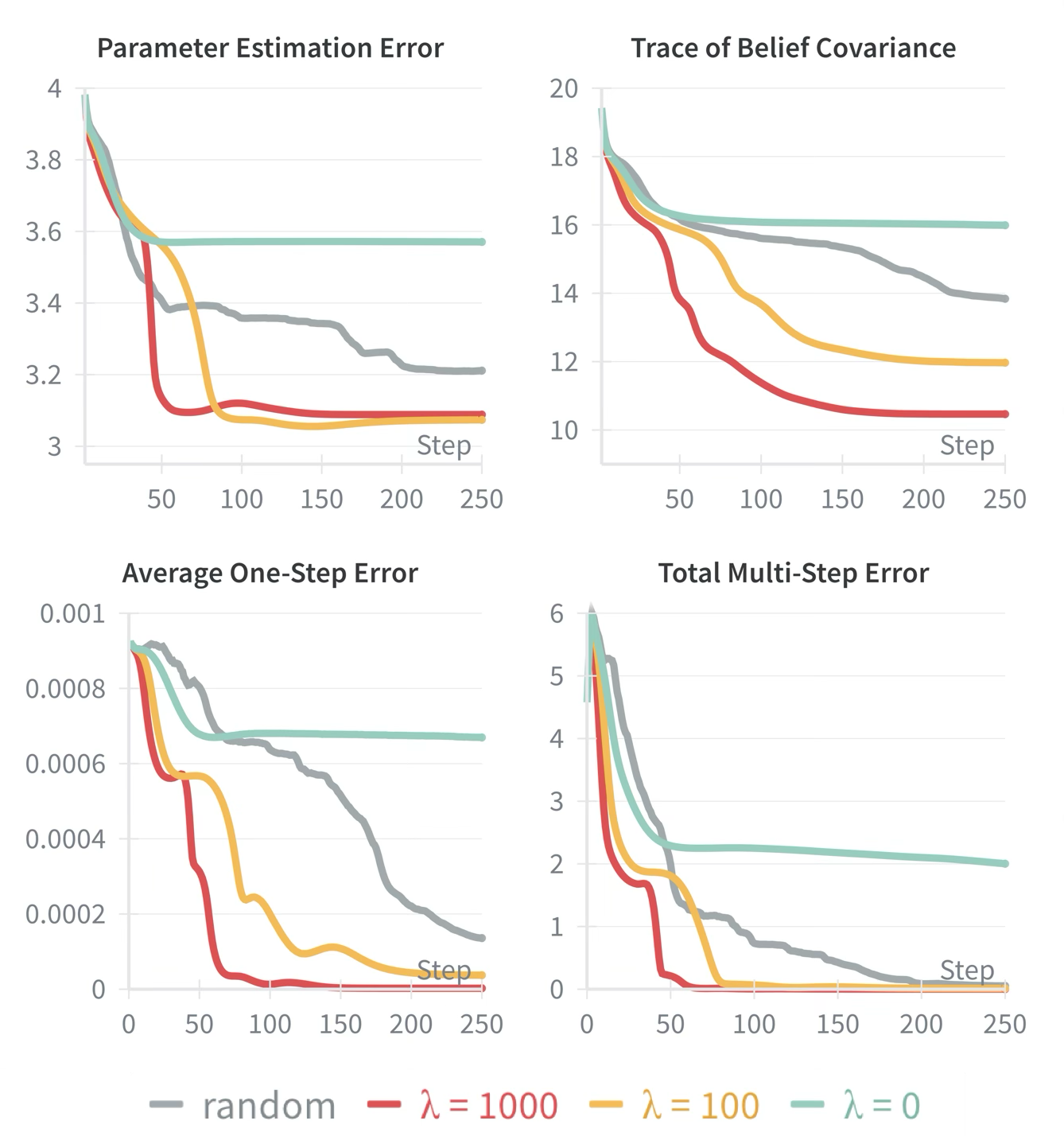}
\caption{PE with LQR policy}
\label{fig:results_lqr}
\end{minipage}
\hfill
\begin{minipage}{0.48\linewidth}
\centering
\includegraphics[width=\linewidth]{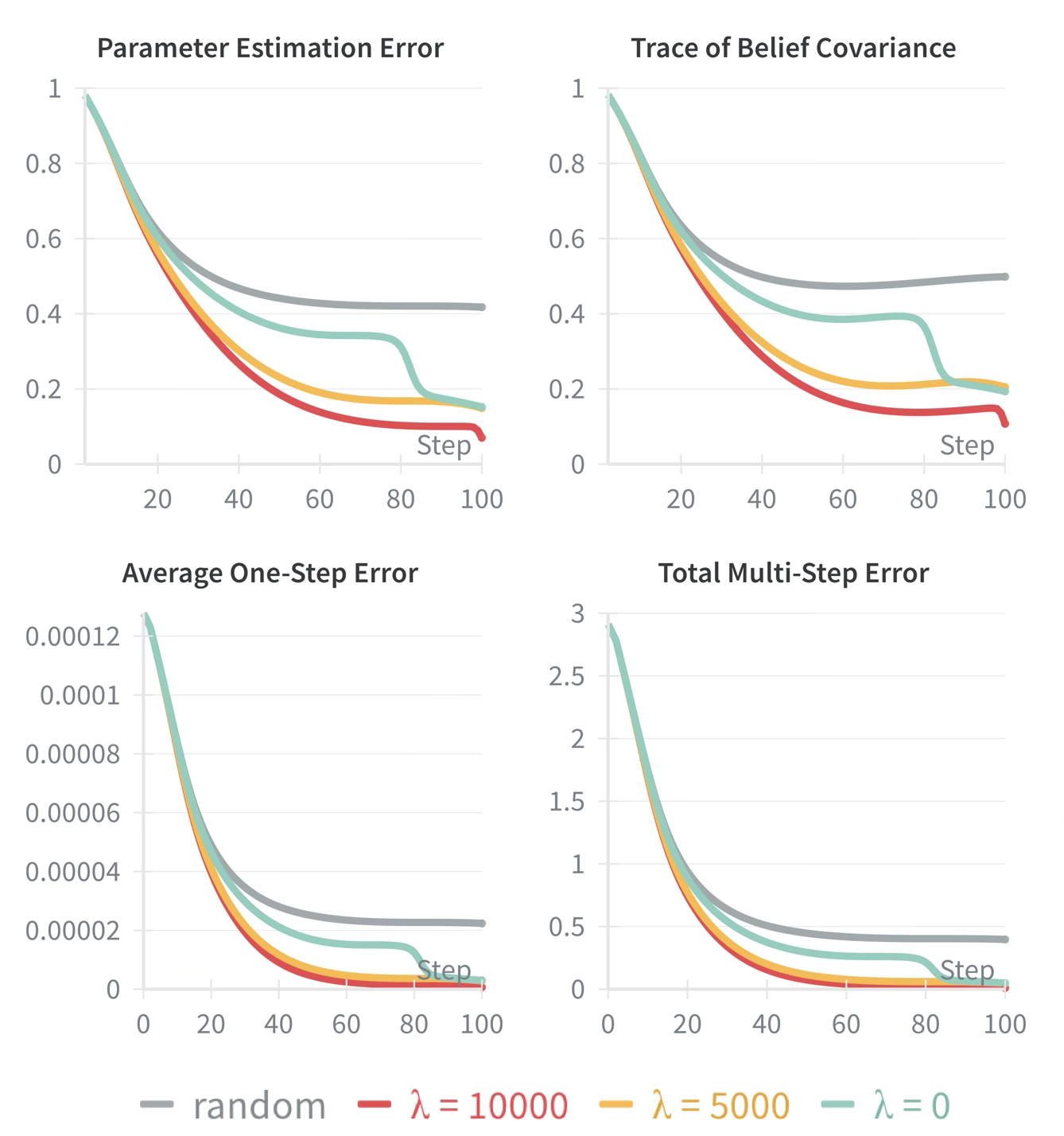}
\caption{PE with differentiable MPC}
\label{fig:results_unicycle}
\end{minipage}
\end{minipage}
\vspace{-30pt}
\end{figure}

\textbf{Experiment 1: Single-agent linear system.} 
We consider an intentionally simplified setting in which a single agent operates under planar double integrator dynamics. The state is $\state_\timeindex = [p_x, p_y, v_x, v_y]^\top$ (position and velocity in 2D) and the control is $\control_\timeindex = [a_x, a_y]^\top$ (acceleration). The dynamics are $\state_{\timeindex+1} = \paramsA \state_{\timeindex} + \paramsB \control_{\timeindex}$ with unknown matrices $\paramsA \in \mathbb{R}^{4 \times 4}$ and $\paramsB \in \mathbb{R}^{4 \times 2}$. The task cost penalizes deviation from a goal state. Results are shown in \cref{fig:results_linear}.

\textbf{Experiment 2: Single-agent damped pendulum.} 
We consider a setting with more sophisticated nonlinear dynamics, in which a single agent controls a damped pendulum. The state is $\state_\timeindex = [\phi_\timeindex, \dot{\phi}_\timeindex]^\top$ (angle and angular velocity) and the control is $\control_\timeindex \in \mathbb{R}$ (applied torque). The dynamics are
\begin{equation*}
    \state_{\timeindex+1} = \begin{bmatrix}
        \phi_\timeindex + \dot{\phi}_\timeindex \Delta t \\
        \dot{\phi}_\timeindex + \ddot{\phi} \Delta t
    \end{bmatrix}
\end{equation*}
where $\ddot{\phi} = (\control_\timeindex - b\dot{\phi}_\timeindex - mgl\sin(\phi_\timeindex))/L$, with unknown damping coefficient $b$ and moment of inertia $L$. The task cost penalizes error in tracking a given reference angle. Results are shown in \cref{fig:results_pendulum}.

\textbf{Experiment 3: Two-agent pursuit-evasion with pursuer LQR policy.} 
We model a two-agent planar pursuit-evasion (PE) game where the pursuer follows an LQR policy with unknown cost matrices. Both agents have double integrator dynamics with states $\state_\timeindex^1, \state_\timeindex^2 \in \mathbb{R}^4$ (evader and pursuer positions/velocities) and controls $\control_\timeindex^1, \control_\timeindex^2 \in \mathbb{R}^2$ (accelerations). The full dynamics are
\begin{equation*}
    \state_{\timeindex+1} = \begin{bmatrix}
        \paramsA \state^1_\timeindex + \paramsB \control^1_\timeindex \\
        \paramsA \state^2_\timeindex + \paramsB \policy^2_\timeindex
    \end{bmatrix}
\end{equation*}
where the pursuer policy is $\policy^2_\timeindex = -(\paramsR + \paramsB^\top \paramsQ \paramsB)^{-1} \paramsB^\top \paramsQ (\paramsA \state^2_\timeindex - \state^1_\timeindex)$ with unknown (to agent 1) cost matrices $\paramsQ \in \mathbb{R}^{4 \times 4}$ and $\paramsR \in \mathbb{R}^{2 \times 2}$. The evader's task cost maximizes distance from the pursuer. Since dynamics for both agents are linear and assumed to be known, we can compute a closed-form LQR policy for the pursuer. Results are shown in \cref{fig:results_lqr}.

\textbf{Experiment 4: Two-agent pursuit-evasion with differentiable MPC policy.}
We extend the pursuit-evasion scenario to a pursuer with nonlinear unicycle dynamics. 
The evader maintains double integrator dynamics ($\state^1_\timeindex \in \mathbb{R}^4$, $\control^1_\timeindex \in \mathbb{R}^2$), while the pursuer has unicycle dynamics with state $\state^{\timeindex\top} = [p_{x,i}^2, p_{y,i}^2, \phi_i^2, v_i^2]$ (position, heading, speed) and control $\control^2_\timeindex = [\omega^2_i, a_i^2]^\top$ (angular velocity, acceleration). 
The pursuer dynamics are:
\begin{equation*}
    \state^2_{\timeindex+1} = \begin{bmatrix} 
        p_{x,\timeindex}^2 +  v_\timeindex^2 \cos(\phi_\timeindex^2)\Delta t \\ 
        p_{y,\timeindex}^2 +  v_\timeindex^2 \sin(\phi_\timeindex^2)\Delta t \\
        \phi_\timeindex^2 +  \omega_\timeindex^2 \Delta t\\
        v_\timeindex^2 +  a_\timeindex^2 \Delta t
    \end{bmatrix}
\end{equation*}
The pursuer policy is encoded by a nonlinear model predictive control problem with a cost parameter $w$ (unknown to the evader): $\policy^2_\timeindex = \arg\min_{\mathbf{u}} w\sum_{i=\timestart+1}^{\timestart + \horizon_\policy}\|\state^1_\timestart - \state^2_i\|^2 + \|\control_i\|^2$, implemented via differentiable optimization in \texttt{JAX} \cite{bradbury2018jax}. This experiment highlights how our framework can extend to a broad class of unknown dynamical models, including those with embedded optimization problems and games \cite{agrawal2019differentiable, amos2017optnet, blondel2022efficient, JMLR:TorchOpt, liu2023learning, liu2024auto, palafox2024smooth}. Results are shown in \cref{fig:results_unicycle}.

\subsection{Discussion}
Across all experiments in \cref{fig:results_linear,fig:results_pendulum,fig:results_lqr,fig:results_unicycle}, increasing the information-gathering weight $\lambda$ (yellow and red lines) reduces parameter estimation error and improves generalization to held-out data, compared to the random and passive learning baselines.
We remark that in Experiment 4 (\cref{fig:results_unicycle}), the passive learning trajectory ($\lambda = 0$) performs better than the random baseline. 
We conjecture this is because, in certain settings, local minima for the task and information-gathering costs lie in the same regions of state space. 
Therefore, passive learning results in enough excitation for accurate parameter identification.

The results confirm both of our hypotheses: active information gathering accelerates parameter learning and improves model generalization compared to passive observation.
The experiments also validate the claim in \cref{sec:contribution3} that the mutual information cost from \Cref{thm:mi_special_case} directly accelerates parameter learning.
Finally, they showcase how our framework applies across single-agent and multi-agent problems with both linear and nonlinear dynamics.

Our framework's modularity enabled rapid instantiation of these diverse experiments: we applied the same EKF belief updater and mutual information cost with linear dynamics, nonlinear dynamics, and embedded differentiable model predictive control modules encoding unknown opponent behavior, all without any problem-specific cost derivations.
This validates the practical utility of our formalization beyond theoretical generality.

\contribution[Experimental validation in single and multi-agent settings.]{
Our experiments showcase the practical utility of our framework by demonstrating information-gathering behavior in a variety of dynamical systems (single-agent, multi-agent, linear, non-linear) without problem-specific cost derivations.}

\section{Conclusion}

In this work, we present a modular framework for sequential decision making under dynamics uncertainty that decouples modeling choices (e.g., dynamics model, belief updaters, etc.) from the information-gathering cost by cleanly exposing the causal dependencies between parameters, beliefs, controls, and observations (\cref{sec:contribution1}). 
We leverage this framework to derive a general information-gathering cost based on directed information, which, in contrast with existing work, is agnostic to most modeling choices (\cref{sec:contribution2}).
In \Cref{thm:mi_special_case} we prove that under certain assumptions, this cost reduces to a mutual information cost that is commonly used in the literature;
we establish a theoretical equivalence between this mutual information cost and the information gain in linearized Bayesian estimation, providing theoretical justification for its use (\cref{sec:contribution3}).
Our experiments illustrate the practical and theoretical value of our work by showing how an instantiation of our proposed framework can replicate results in existing literature.

Collectively, our contributions offer a unified framework for information gathering under dynamics uncertainty. 
The causal structure we expose, and the directed information cost we derive from it, provides a principled foundation that practitioners can instantiate across diverse modeling choices.

\begin{credits}
\subsubsection{\ackname} We would like to thank Brett Barkley for formatting advice. This research was sponsored by the Army Research Laboratory under cooperative agreement number W911NF-25-2-0021, and by the National Science Foundation under grant number 2409535.

\end{credits}

\bibliographystyle{splncs04}
\bibliography{example}  

\end{document}